\documentclass[preprint,12pt]{elsarticle}

\usepackage{amssymb,algorithm,algorithmic,graphicx,color}

\newcommand{\cc}[2]{{#2}} 

\newcommand{\cR}[1]{\cc{R}{#1}} 
\newcommand{\cA}[1]{\cc{A}{#1}} 
 
\newcommand{\cW}[1]{\cc{W}{#1}}

\newcommand{\scst}{\scriptscriptstyle}

\newcommand{\Z}{\mathbb{Z}} 
\newcommand{\N}{\mathbb{N}} 
\newcommand{\R}{\mathbb{R}} 

\newtheorem{definition}{Definition}[section]
\newtheorem{lemma}{Lemma}[section]
\newtheorem{prop}[lemma]{Proposition}
\newtheorem{corollary}[lemma]{Corollary}

\newtheorem{remark}[lemma]{Remark}

\journal{Computer Vision and Image Understanding}

\begin{document}

\begin{frontmatter}

\title{Invariant Representative Cocycles of Cohomology Generators using Irregular Graph Pyramids}
\author{Rocio Gonzalez-Diaz$^{a}$, Adrian Ion$^{b,d}$,\\ Mabel Iglesias-Ham$^{b,c}$, and  Walter G. Kropatsch$^{b}$}
\address{
$^{a}$ Applied Math Department, School of Computer Engineering\\
 University of Seville,  Reina Mercedes Avenue,\\ 
 CP: 41012 Seville, Spain\\
rogodi@us.es\\

$^{b}$ Pattern Recognition and Image Processing Group, Vienna
University of Technology,\\
Faculty of Informatics, Institute of Computer Aided Automation,\\
Favoritenstr. 9/1832, A-1040 Vienna, Austria.\\
\{ion,mabel,krw\}@prip.tuwien.ac.at\\

$^{c}$ Pattern Recognition Department, Advanced Technologies
Application Center,\\
7th Avenue \#21812 \%218 and 222, Siboney Neighborhood, \\
Playa, C.P. 12200, Havana City, Cuba. \\
miglesias@cenatav.co.cu \\

$^{d}$ Institute for Numerical Simulation\\
Faculty of Mathematics and Natural Sciences, University of Bonn,\\
Wegelerstr. 6, 53115 Bonn, Germany.\\
%
}

\begin{abstract}
Structural pattern recognition describes and classifies data based on the relationships of features and parts. Topological invariants, like the Euler number, characterize the structure of objects of any dimension. Cohomology can provide more refined algebraic invariants to a topological space than does homology. It assigns `quantities' to the chains used in homology to characterize holes of any dimension. Graph pyramids can be used to describe subdivisions of the same object at multiple levels of detail. This paper presents cohomology in the context of structural pattern recognition and introduces an algorithm to efficiently compute representative cocycles (the basic elements of cohomology) in 2D using a graph pyramid. 
An extension to obtain scanning and rotation invariant cocycles is given.
\end{abstract}

\begin{keyword}graph pyramids, representative cocycles of cohomology generators\end{keyword}
\end{frontmatter}

\section{Introduction}

Image analysis deals with digital images as input to pattern recognition systems \cA{with the purpose to extract information about their content, usually objects. Objects appear in images affected by transformations (e.g. rotation, zoom, projection) and noise.}
Topological features have the ability to ignore changes in \cA{the geometry of objects}  \cA{by extracting object properties invariant to elastic transformations}. Simple \cA{topological} features are for example the number of connected components, the number of holes, etc., while more refined ones, like homology and cohomology, characterize holes and their relations.

An example application of topological features is topology simplification,  an active field in geometric modeling and medical imaging where high-resolution surfaces are created through iso-surface extraction from volumetric representations, obtained by 3D photography, CT, or MRI. Iso-surfaces often contain many topological errors in the form of tiny handles. These nearly invisible artifacts hinder subsequent operations like mesh simplification, re-meshing, and parametrization. See, for example~\cite{WHDS04}. \cA{Another application is shape description and matching, where persistence and homology of a function defined on a shape have been successfully applied to extract shape features~\cite{Allili2007}.}

\cW{A 2D image is the result of projecting a 3D scene into the image plane.
Often the precise camera parameters are not known and still humans have
no problem in correctly interpreting the displayed objects in the image.
A 3D object is surrounded by a reflecting surface which itself may split into
several smaller but connected patches which can be characterized by their
color, their texture or other visual properties. The visible part
of the object's surface maps into a region of the image which shows
the same adjacencies of patches as the original surface
(because {\it the camera sees the same side of the surface})
although the geometry of the patches may change due to projection,
due to camera or object movement or due to deformations of the object.
Sometimes a collection of patches is completely surrounded
by some other surface patches,
e.g. a fancy soccer ball on a Spanish T-shirt \cR{(see Fig. \ref{balonencamiseta}). }
Although both the picture of the ball as well
as the T-shirt may have a specific patch structure 
\cW{(stripes on shoulder and arms, logo)}
it is clear which subset of regions forms the ball and which regions
belong to the remaining T-shirt. A simplified version could describe
the pixels of the T-shirt 1's and the ball's pixels by 0 \cR{(see Fig. \ref{balonencamiseta}). } Then
the ball is a {\it hole} in the T-shirt as long as the surrounding
of the ball is visible and not occluded by other objects.
It may be highly difficult to uniquely and reliably identify any
of the involved small patches individually under difficult geometric
deformations while the overall arrangement of patches forming the ball
is mostly invariant to these geometric deformations. How to segment patches
into meaningful aggregations has been dealt with in many other segmentation
methods and it is not the main emphasis of this paper. We therefore restrict
ourselves in the following on binary images with the understanding
that each region may be the collection of several subregions belonging
together. }

\begin{figure}[t!]
\begin{center}\includegraphics[width=8cm]{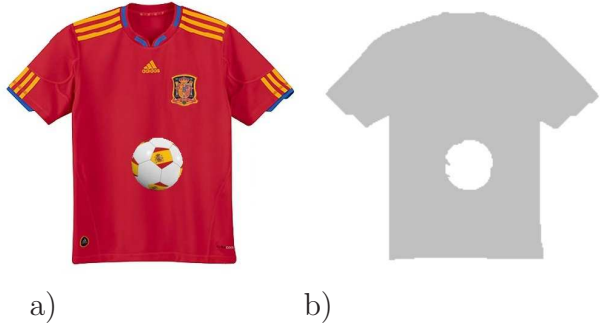}
$\mbox{ }$\\
a) \hspace{3cm} b) \hspace{3cm} $\mbox{ }$ \\
\end{center}
\caption{a) Original image: ball on T-shirt of world champions 2010 in soccer; b) a segmentation of the original image using pyramids.}
\label{balonencamiseta}
\end{figure}

Considering 2D binary images, an \begin{it}object\end{it} is defined by a connected set (4-connectivity) of foreground pixels. 
A \textit{region adjacency graph} (RAG) encodes the adjacency of regions in a partition. 
The holes in a RAG  associated to an object of a 2D binary digital image can be characterized by
establishing an equivalence between all the cycles  as follows:
two cycles are equivalent if one can be obtained from the other by joining to it one or more \textit{degenerate cycles}
(cycles with exactly $4$ edges). For example, there is only one equivalence class for the foreground (gray pixels) of the digital image in
Fig.~\ref{fig:example1}, which represents the unique white hole.
This is similar to considering the digital image as a cell complex\footnote{Intuitively a cell complex is defined by a set of $0$-cells (vertices) that bound a set of $1$-cells (edges), that bound a set of $2$-cells (faces), etc.} \cite{Hatcher02a} (see Fig.~\ref{fig:example1}.c).
\cW{Unfortunately digital images are not `clean', noise can create many
unwanted holes which complicate the correct interpretation.}
One can ask for the edges we have to delete in order to `destroy' a hole.
In the example in Fig.~\ref{fig:example1}, it is not enough to delete only one edge.
The deletion of the bold edges in Fig.~\ref{fig:example1}.c  
 together with the faces that they bound produces the `disappearing' of the hole.
\cA{The set of bold edges in Fig.~\ref{fig:example1}.c define a $1$-cocycle, a topological invariant of the respective object. Equivalence classes of such cocycles are the elements of \textit{cohomology}.}

\cW{To cope with complexity issues that arise when highly complex algorithms
must be applied to huge amounts of data, graph pyramids are used.
These hierarchical data structures offer possibilities to reduce the
amount of data by local operations which can be applied in parallel and
which have the enormous advantage to preserve the topological properties
of the data. Hence the search for independent cocycles can be correctly 
done on a fraction of the data at the top of the pyramid. The simplified
geometry of these cocycles can then  \cR{be delineated} top-down through
the levels of the pyramid by again local processes up to the high
accuracy of the base level. }

\cA{Maybe due to its more abstract nature, lacking a geometric meaning and due to its computation complexity,} cohomology has not been yet widely applied to pattern recognition and image processing. \cA{This paper is possibly the first attempt} to use it in the context of digital images. For this \cA{purpose, in this paper we consider} the best known environment which are 2D images whereas nD is the \cA{ultimate} goal.
Concepts related to cohomology can have associated interpretations in graph theory. Having these interpretations opens the door for applying classical \cA{efficient} graph theory algorithms to compute and manipulate these features.
Besides, for objects embedded in $\mathbb{R}^3$, 
homology\cA{-- a wider used topological invariant,} and cohomology groups are isomorphic. But
the ring structure presented in cohomology characterizes the relations between $2$-holes (cohomology generators of \cA{dimension} $1$), which homology does not.
Indeed, dealing with homology and cohomology properties, representative cycles and cocycles, and their computation is quite different and doing this study in 2D gives important insights which should be relevant for extension to \cA{$n$D}, $n>2$.
Initial results regarding this work have been presented in~\cite{GonzalezDiaz09a}. The current paper extends \cA{our earlier}  \cA{publication} with detailed insights and proofs, and a refinement of the previous method that makes the obtained cocycles scanning and rotation invariant \cA{in the case of an identical discretization}.

The paper is organized as follows: Sections~\ref{sec:pyr} and~\ref{sec:cohom} recall graph pyramids and cohomology, and make initial connections. In Section~\ref{sec:preserving-topology}, preserving-topology properties in irregular graph pyramids are given. Section~\ref{sec:computation} presents the proposed method. 
Section~\ref{sec:stable-cocycles} uses \cA{the} properties of the \cA{proposed} method to extend it and obtain scanning and rotation invariant cocycles. Section~\ref{sec:conc} concludes the paper.

\begin{figure}[t!]
\begin{center}\includegraphics[width=8cm]{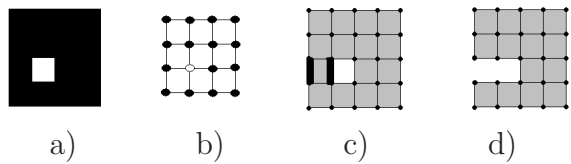}
$\mbox{ }$\\
\hspace{1.3cm} a) \hspace{1.3cm} b) \hspace{1.3cm} c) \hspace{1.3cm} d) \hspace{1.3cm}$\mbox{ }$ \\
\end{center}
\caption{a) A 2D digital image $I$; b) its RAG; c) a cell complex associated to $I$ (with bold edges, a representative cocycle); and d) the cell complex without the hole.}
\label{fig:example1}
\end{figure}

\section{Irregular Graph Pyramids}
\label{sec:pyr}

\cA{A graph (see for example~\cite{Diestel97a}) is an ordered pair $G = (V,E)$ comprising a set $V$ of vertices and a set $E$ of edges. Each edge $e \in E$ is incident to two not necessarily distinct vertices $v, w \in V$, written as $e=(v, w)$. An edge $e=(v, w)$ is said to be directed if the pair $(v, w)$ is an ordered pair. An edge is said to be a self-loop if $v = w$. The edge $e=(v, w)$ is called a parallel edge iff $\exists e' \in E, e' \neq e$ s.t. $e' = (v, w)$. A graph is called undirected if none if its edges is directed, and it is called planar if it can be drawn in a plane with no edges crossing (vertices are drawn as points and edges as lines connecting their incident vertices). Given a graph $G = (V, E)$ removing an edge $e \in E$ will result in the graph $G' = (V, E \setminus \{e\})$, contracting the edge $e=(v,w)$ implies removing it and identifying its incident vertices s.t any remaining edge previously incident to $v$ or $w$ is now incident to the unique vertex $v=w$.}

\cA{Given a decomposition of an object or image into regions a} region adjacency graph (RAG) \cA{is an undirected graph that} encodes the adjacency of regions in a partition. A vertex is associated to each region, vertices of neighboring regions are connected by an undirected edge. Classical RAGs do not contain any self-loops or parallel edges. An \textit{extended region adjacency graph} (eRAG) is a RAG that contains the so-called \textit{pseudo edges}, which are self-loops and parallel edges used to encode neighborhood relations to a cell completely enclosed by one or more other cells~\cite{K95}. The \textit{dual} graph of an eRAG $G$ is called a \textit{boundary graph} (BG) and is denoted by $\bar{G}$ ($G$ is said to be the \textit{primal} graph of $\bar{G}$). The edges of $\bar{G}$ represent the boundaries (borders) of the regions encoded by $G$, and the vertices of $\bar{G}$ represent points where boundary segments meet. $G$ and $\bar{G}$ are planar graphs. There is a one-to-one correspondence between the edges of $G$ and the edges of $\bar{G}$, which induces a one-to-one correspondence between the vertices of $G$ and the 2D cells (will be denoted by \textit{faces}\footnote{Not to be confused with the vertices of the dual of a RAG (sometimes also denoted by the term \textit{faces}).}) of $\bar{G}$. The dual of $\bar{G}$ is again $G$. The following operations are equivalent: edge contraction in $G$ with edge removal in $\bar{G}$, and edge removal in $G$ with edge contraction in $\bar{G}$. 

A (dual) irregular graph pyramid\cR{~\cite{Kropatsch93,Willersinn94,K95,KHPL05}} is a stack of successively reduced planar graphs $P = \{(G_0,\bar{G}_0),\dots,(G_n,\bar{G}_n)\}$. Each level $(G_k,\bar{G}_k)$, $0 < k \leq n$, is obtained by first contracting edges in $G_{k-1}$ (removal in $\bar{G}_{k-1}$), if their end vertices have the same label (regions should be merged), and then removing edges in $G_{k-1}$ (contraction in $\bar{G}_{k-1}$) to simplify the structure. The contracted and removed edges are said to be \textit{contracted} or \textit{removed} (sometimes called \textit{removal} edges) in $(G_{k-1},\bar{G}_{k-1})$. In each $G_{k-1}$ and $\bar{G}_{k-1}$, contracted edges form trees called \textit{contraction kernels}. One vertex of each contraction kernel is called a \textit{surviving vertex} and is considered to have  `survived' to $(G_k,\bar{G}_k)$. The vertices of a contraction kernel in level $k-1$ form the \textit{reduction window} of the respective surviving vertex $v$ in level $k$. The \textit{receptive field} of $v$ is the (connected) set of vertices from level $0$ that have been `merged' to $v$ over levels $0,\dots, k$. The \textit{equivalent contraction kernel (ECK)} of a vertex $v$ is the tree obtained by replacing $v$ and all is descendants with their corresponding \cA{contraction kernels}. The vertices of the ECK of $v$ form the receptive field of $v$.

\begin{figure}[t]
\centerline{\includegraphics[width=10cm]{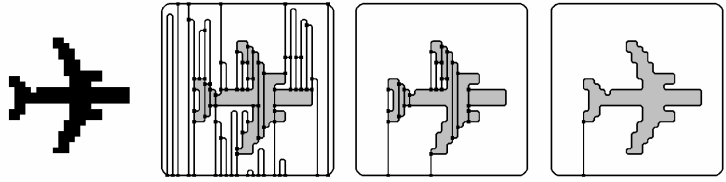}}
\vspace{-2mm}
\caption{A digital image $I$, and boundary graphs $\bar{G}_6$, $\bar{G}_{10}$ and $\bar{G}_{15}$ of the pyramid of $I$.}
\label{img:pypramid}
\end{figure}

Algorithm~\ref{alg:buildpyr} gives the main steps used to build a graph pyramid. In Line~\ref{lin:contract-and-simplify} the operations are performed on both graphs (adjacency and boundary) even if for simplicity only one of the graphs is mentioned in each step. Line~\ref{lin:contract-and-simplify}.ii removes self\cR{-}loops bounding an empty face in the adjacency graph. Line~\ref{lin:contract-and-simplify}.iii simplifies the region boundaries in the boundary graph.

\begin{algorithm}[tbph]
\caption{Build (dual) graph pyramid}\label{alg:buildpyr}
\label{alg:pyramid}
\textbf{Input}: image $I$ /*\textit{pixels labeled `object' or `background'}*/
\begin{algorithmic}[1]
\STATE $(G_0,\bar{G}_0) = ((V_0, E_0), (\bar{V}_0, \bar{E}_0))$ \\ \COMMENT{$V_0$ associates a vertex to every pixel. $E_0$ connects vertices corresponding to 4-connected pixels. $G_0$ and $\bar{G}_0$ are dual.}
\STATE $k = 0$
\REPEAT
\STATE \COMMENT {select edges to contract} \label{lin:boru}\\
$T = \emptyset$ \\
\textbf{for all }$ v \in G_k$ \textbf{do}\\
\quad i. select an edge $(v,w) \in G_k$ with $v,w$ having the same label \\
\quad ii. $T \leftarrow T \cup (v,w)$ \COMMENT{add edge}
\STATE \textbf{if} $T \neq \emptyset$ \textbf{then} \label{lin:contract-and-simplify}\\
\quad\COMMENT{region merging, easier described in the adjacency graph:}\\
\quad i. $(G', \bar{G}') \leftarrow$ contract edges $T$ of $G_k$ (removal in $\bar{G}_k$)\\
\quad\COMMENT{simplification, easier described in the boundary graph:}\\
\quad ii.  \cR{$(G'', \bar{G}'') \leftarrow$} contract pending trees in $\bar{G}'$ (removal in $G'$)\\
\quad iii. $(G_{k+1}, \bar{G}_{k+1}) \leftarrow$ \parbox[t]{8.5cm}{contract one distinct edge incident to each vertex of degree 2 in  \cR{$\bar{G}''$ (removal in $G''$)}}
\STATE $k \leftarrow k+1$

\UNTIL $T = \emptyset$
\end{algorithmic}
\textbf{Output}: Graph pyramid $P = \{(G_0,\bar{G}_0),\dots,(G_{k-1},\bar{G}_{k-1})\}$.
\end{algorithm}

\section{Homology, Cohomology and Integral Operators}
\label{sec:cohom}

\cR{We refer to~\cite{Munkres} for an introduction
to homology and cohomology.}

Intuitively, homology characterizes the holes of any dimension (i.e. connected components, 1-dimensional holes, etc.) of an $n$-dimensional object. It defines the concept of \textit{generators} which, for example for 2D objects are similar to closed paths of edges surrounding holes. More general, $k$-dimensional manifolds surrounding $(k+1)$-dimensional holes are generators~\cite{Munkres}, and define equivalence classes of $(k+1)$-holes. Cohomology arises from the algebraic dualization of the construction of homology. It manipulates groups of homomorphisms to define equivalence classes. Intuitively, cocycles (the invariants computed by cohomology), represent the sets of elements (e.g. edges) to be removed to destroy certain holes. See
Fig.~\ref{fig:example1}.c for an example cocycle.\\

\begin{figure}[tb]
\centerline{
\begin{tabular}{cc}
\includegraphics[width=3.3cm]{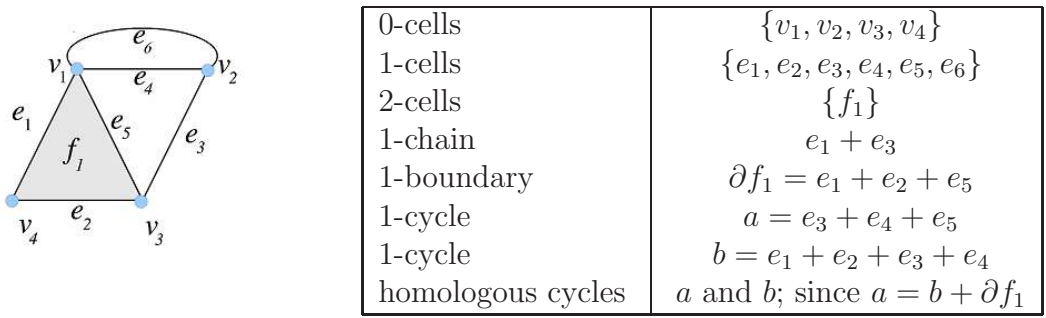}
& \hspace{1cm}
\begin{tabular}{|l|c|}\vspace{-3cm}\\
\hline
$0\mbox{-cells}\;$ & $\;\{v_1,v_2,v_3,v_4\}$\\
$1\mbox{-cells}\;$ & $\;\{e_1,e_2,e_3,e_4,e_5,e_6\}$\\
$2\mbox{-cells}\;$ & $\; \{f_1\}$\\
$1\mbox{-chain}\;$ & $\; e_1+e_3$\\
$1\mbox{-boundary}\;$ & $\; \partial f_1 =e_1+e_2+e_5$\\
$1\mbox{-cycle}\;$ & $\; a= e_3+e_4+e_5$\\
$1\mbox{-cycle}\;$ & $\; b=e_1+e_2+e_3+e_4$\\
$\mbox{homologous cycles}\;$ & $\; a$ and $b$; since $a=b+\partial f_1$\\\hline
\end{tabular}
\end{tabular}
}
\caption{Example of cell, chain, boundary and cycle.}
\label{img:exSimplComplex}
\end{figure}

A {\em homeomorphism} is a bijective continuous function between two spaces, that has a continuous inverse function.
 They are the mappings which preserve all the topological properties  of a given space. Two spaces with a homeomorphism between them are called {\em homeomorphic}, and from a topological viewpoint they are the same.

Two continuous functions from one topological space to another are called {\em homotopic} if one can be `continuously deformed' into the other.
Two spaces $X$ and $Y$ are {\em homotopy equivalent} if there are maps $f:X\to Y$ and $g:Y\to X$ such that $gf$ is homotopic to $id_{\scst X}$ and 
$fg$ is homotopic to $id_{\scst Y}$. Observe that if two spaces are homeomorphic then they are homotopic.
A {\em homotopy invariant} is a topological property which is invariant under homotopy.

A $p$-dimensional {\em cell} (or $p$-cell, for short) is a topological space that is homeomorphic to the $p$-dimensional ball \cR{$B^p$}. 
A $0$-cell is homeomorphic to a point, a $1$-cell to an arc and a $2$-cell  to a disk.
Roughly speaking, a {\em cell- (or CW-) complex} is built by gluing together the basic building blocks called cells.

Figures~\ref{img:exSimplComplex} and \ref{img:excohom} illustrate the following abstract concepts.

\subsection{Homology}

The notion of $p$-\textit{chain} is defined as a formal sum
of $p$-cells.
The chains are considered over $\mathbb{Z}/2$ coefficients i.e. a $p$-cell is either present in a $p$-chain (coefficient 1) or absent (coefficient 0) - any cell that appears twice vanishes. 
The set of $p$-chains form an abelian  group called the \textit{$p$-chain group} $C_p$.
This group is generated by all the $p$-cells.
The \textit{boundary operator} is a set of homomorphisms $\{\partial_p:C_p\to C_{p-1}\}_{p \geq 0}$ connecting two \cR{consecutive} dimensions.
By linearity, the boundary of any
$p$-chain is  defined as the formal sum of the boundaries of each $p$-cell
that appears in the chain. The boundary of $0$-cells (i.e. points) is always $0$.
A {\em chain complex} is the set of all the  chain groups connected by the boundary operator: $\cdots\stackrel{\partial_{p+1}}{\to} C_{p}\stackrel{\partial_p}{\to} C_{p-1}{\to} \cdots\stackrel{\partial_1}{\to} C_0
\stackrel{\partial_0}{\to} 0$.

\begin{figure}[tb]
\centerline{
\begin{tabular}{cc}
\includegraphics[width=3.3cm]{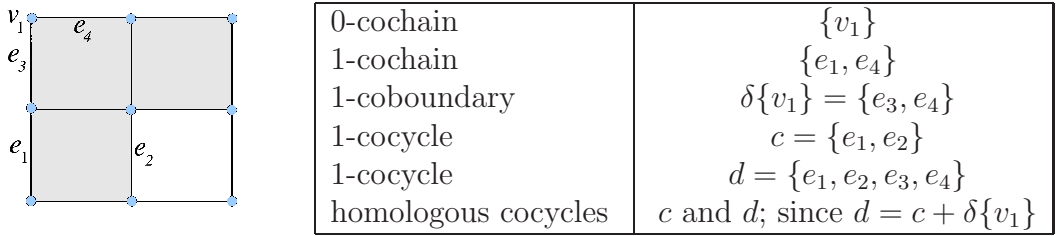}
& \hspace{0.3cm}
\begin{tabular}{|l|c|}\vspace{-2.8cm}\\
\hline
$0\mbox{-cochain}\;$ & $\;  \{v_1\}$\\
$1\mbox{-cochain}\;$ & $\;  \{e_1,e_4\}$\\
$1\mbox{-coboundary}\;$ & $\; \delta \{v_1\} =\{e_3,e_4\}$\\
$1\mbox{-cocycle}\;$ & $\; c= \{e_1,e_2\}$\\
$1\mbox{-cocycle}\;$ & $\; d=\{e_1,e_2,e_3,e_4\}$\\
$\mbox{homologous cocycles}\;$ & $\; c$ and $d$; since $d=c+\delta \{v_1\}$\\\hline
\end{tabular}
\end{tabular}
}
\caption{Example of cochain, cocycle and coboundary.}
\label{img:excohom}
\end{figure}

A $p$-chain $\sigma$ is called a
$p$-\textit{cycle} if $\partial_p  (\sigma)=0$.
If $\sigma=  \partial_{p+1} (\mu)$ for some $(p+1)$-chain $\mu$ then $\sigma$
 is called a $p$-\textit{boundary}. Two $p$-cycles $a$ and $a'$ are \textit{homologous} if there exists a $p$-boundary $b$ such that $a=a'+b$. 
Denote the groups of $p$-cycles
and $p$-boundaries by $Z_p$ and $B_p$ respectively.
\cR{For each $p$, $\partial_{p-1}\partial_{p}=0$. In other words,} all $p$-boundaries are $p$-cycles ($B_p\subseteq Z_{p}$).
Define the  $p^{th}$  \textit{homology group} to be the quotient group $H_p=Z_p/B_p$, for all $p$. 
Each element of $H_p$ is a class obtained by adding each $p$-boundary to a given $p$-cycle $a$. 
Then $a$ is a \textit{representative $p$-cycle} of the homology class $a+B_p$.

\cR{Since the chains are considered over $\mathbb{Z}/2$ coefficients, the chain groups are vector spaces and the boundary operators are linear operators. The cycle and and boundary groups are just the kernel and image of such operator. The homology group is a quotien space.} 

\subsection{Cohomology}

\textit{Cohomology groups} are constructed by turning chain groups into groups of homomorphisms and
boundary operators into their dual homomorphisms.
Define a $p$-\textit{cochain} as a homomorphism $c : C_p \to \mathbb{Z}/2$.
We can see a $p$-cochain as a binary mask of the set of $p$-cells: imagine you order all $p$-cells in the complex. (let's say we have $n$ $p$-cells, and call this ordered set $S_p$). Then a $p$-cochain $c$ is a binary mask of $n$ values in $\{0,1\}^n$, where $n$ is the number of $p$-cells in the complex.
When no confusion can arise, we will identify the $p$-cochain $c$ with the set $S$ of $p$-cells that are evaluated  to $1$ by $c$.

The $p$-cochains form the set $C^p$ which is a group.
The boundary operator defines a dual set of homomorphisms, the \textit{coboundary operator}
$\{\delta^p: C^p\to C^{p+1}\}_{p\geq 0}$, such that $\delta^p(c)=c\partial_{p+1}$ for any  $p$-cochain $c$.
Since the coboundary operator runs in a direction
opposite to the boundary operator, it raises the dimension. Its kernel is the
group of \textit{cocycles} and its image is the group of \textit{coboundaries}.
A $p$-cochain $c$ is a $p$-cocycle if $\delta^p c(=c\partial_{p+1}):C_{p+1}\to \mathbb{Z}/2$ is the null homomorphism. A $p$-cochain $d$ is a $p$-coboundary 
if there exists a $(p-1)$-cochain $e$ such that $d=\delta^{p-1}e(=e\partial_p)$.
Two $p$-cocycles $c$ and $c'$ are \textit{cohomologous} if there exists a $p$-coboundary $d$ such that
 $c=c'+d$. The $p^{th}$  \textit{cohomology group} is defined as the quotient of $p$-cocycle modulo
$p$-coboundary groups, $H^p = Z^p/B^p$, for all $p$.
Each element of $H^p$ is a class obtained by adding each $p$-coboundary to a given $p$-cocycle $c$. 
Then $c$ is a \textit{representative $p$-cocycle} of the cohomology class $c+B^p$, denoted by $[c]$.

\begin{figure}[t!b]
\centerline{
\includegraphics[width=12cm]{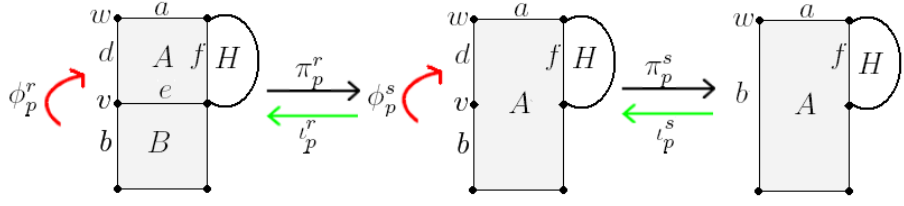}}
\caption{The cell complex $K$,  $K'$ and $K''$;  and the homomorphisms $\phi_p^r, \pi_p^r, \iota_p^r$ and 
$\phi_p^s, \pi_p^s, \iota_p^s$.}
\label{img:relations}
\end{figure}

 \begin{definition}
A set of $p$-cocycles $\{c_1,\dots,c_n\}$ is a \cR{basis} of representative $p$-cocycles if:
\begin{itemize}
\item  Any other $p$-cocycle $c$ can be written as a linear combination of the $p$-cocycles of the set plus a $p$-coboundary, that is:
$$c=\sum_{i=1}^n \lambda_i c_i +\delta^{p-1} e$$
where $\lambda_i=0,1$ and $e$ is a $(p-1)$-cochain.

\item  None of the $p$-cocycles in the set can be written as a linear combination of the rest plus a $p$-coboundary (minimality).
\end{itemize}

\end{definition}

\begin{remark}
If $\{c_1,\dots,c_n\}$ is a \cR{basis} of representative $p$-cocycles then $\{[c_1],$ $\dots,$ $[c_n]\}$ is a \cR{basis} of the $p^{th}$ cohomology group of the object. Each $[c_i]$ is a {\em cohomology generator}.
\end{remark}

\subsection{Integral Operator}

\cR{Starting from a chain  complex,}
$\cdots\stackrel{\partial_2}{\to} C_1\stackrel{\partial_1}{\to} C_0
\stackrel{\partial_0}{\to} 0$, take a $q$-cell $\sigma$  and  a $(q+1)$-chain $\alpha$.
An \textit{integral operator} \cite{GJMMR08} is defined as the set of homomorphisms $\{\phi_p:C_p\to C_{p+1}\}_{p\geq 0}$ such that
$\phi_q(\sigma)=\alpha$, $\phi_q(\mu)=0$ if $\mu$ is a $q$-cell different to $\sigma$, and for all $p \neq q$ and any $p$-cell $\gamma$ we have $\phi_p(\gamma)=0$.
It is extended to all $p$-chains by linearity.

An integral operator $\{\phi_p:C_p\to C_{p+1}\}_{p\geq 0}$ satisfies the \textit{chain-homotopy property} iff $\phi_{p}\partial_{p+1}\phi_p=\phi_p$ for each $p$. 
For $\phi_p$ satisfying the chain-homotopy property, define  $\pi_p=id_p+\phi_{p-1}\partial_p+\partial_{p+1}\phi_p$ where $\{id_p:C_p\to C_p\}_{p\geq 0}$
is the identity. 
\cR{Define $im\pi_p=\{b\in C_p \mid \exists a\in C_p \mbox{ s.t. } \pi_p(a)=b\}$.}
Then, 
$\cdots\stackrel{\partial_2}{\to} im\pi_1\stackrel{\partial_1}{\to} im\pi_0 \stackrel{\partial_0}{\to} 0$ is a chain complex and   $\{\pi_p:C_p\to im\pi_p\}_{p\geq 0}$ is a \textit{chain equivalence} \cite{Munkres}. Its chain-homotopy inverse is the inclusion map $\{\iota_p:im\pi_p\to C_p\}_{p\geq 0}$.
 Integral operators satisfying the chain-homotopy property can be seen as a kind of inverse boundary operator: They raise the dimension and satisfy the nilpotent condition
$\phi_{p+1}\phi_p=0$ for all $p$. Although, in general, $\phi_{p-1}\partial_p\neq id_p$ and $\partial_{p+1}\phi_p\neq id_p$, what happens is  $\phi_p\partial_{p+1}\phi_p=\phi_p$ for all $p$ (which would be equivalent to $x\cdot \frac{1}{x}\cdot x=x$ for $x\in \mathbb{R}\setminus\{0\}$).
Consider, for example, the cell complex $K$ in Fig.~\ref{img:relations} on the left. The integral operator
associated to the removal of the edge $e$ is given by  \cR{ $\phi_1^r(e)=A$.}
 Then,  
 \cR{$\pi^r_1(e)=(id_1+\phi_0^r\partial_1+\partial_2\phi_1^r)(e)=e+a+f+d+e=a+f+d$,}
  $\pi^r_2(A)=0$, $\pi^r_2(B)=A$ ($A+B$ is renamed as $A$ in $K'$) and
 $\pi^r_p$ is the identity over the other $p$-cells of $K$, 
 $p=0,1,2$.
The removal of edge $e$ decreases the degree of vertex $v$ allowing for further simplification.

\begin{figure}[t!b]
\centerline{
\begin{tabular}{|c|c|c|}
\hline 
 & $\phi_p^r$ & $\pi_p^r$ \\\hline
$e$ & $B$ & $a+f+d$ \\
$B$ & $0$ & $0$ \\ 
$A$ & $0$ & $A$ \\
other $p$-cell $\sigma$ & $0$ & $\sigma$ 
\\\hline
\end{tabular}
\hspace{1cm}
\begin{tabular}{|c|c|c|}
\hline 
 & $\phi_p^s$ & $\pi_p^s$ \\\hline
$v$ & $d$ & $w$ \\
$d$ & $0$ & $0$ \\ 
$b$ & $0$ & $b$ \\
other $p$-cell $\sigma$ & $0$ & $\sigma$ 
\\ \hline
\end{tabular}
}
\caption{The homomorphisms $\phi_p^r, \pi_p^r$ and $\phi_p^s,\pi_p^s$.}
\label{img:relations1}
\end{figure}

\cR{The following lemma guarantees the correctness of the down projection procedure for computing cocycles given in Section \cR{\ref{sec:computation}.} Since graph pyramids offer possibilities to reduce the amount of data by local operations, then the search for independent cocycles can be  done on the top of the pyramid. Hence, the following lemma guarantees that these cocycles can be correctly delineated top-down through the levels of the pyramid.   
}

\begin{lemma} \label{lem:cocycle-integral-operator}
Let  $\{\phi_p:C_p\to C_{p+1}\}_{p\geq 0}$ be an integral operator satisfying the chain-homotopy property.
The chain complexes $\cdots\stackrel{\partial_2}{\to} C_1\stackrel{\partial_1}{\to} C_0
\stackrel{\partial_0}{\to} 0$ and
$\cdots\stackrel{\partial_2}{\to} im\pi_1\stackrel{\partial_1}{\to} im\pi_0
\stackrel{\partial_0}{\to} 0$ have isomorphic homology and cohomology groups.
If $c:im\pi_p \to \mathbb{Z}/2$ is a representative $p$-cocycle of a cohomology generator,
then $c\pi: C_p\to \mathbb{Z}/2$ is a representative $p$-cocycle of the same generator. 
\end{lemma}

\begin{proof} An  \cR{integral} operator that satisfies the chain homotopy property, is a chain homotopy of the identity $\{id_p:C_p\to C_p\}_{p\geq 0}$
to $\{\iota_p\pi_p:C_p\to C_p\}_{p\geq 0}$. Therefore, $\{\pi_p:C_p\to im\pi_p\}_{p\geq 0}$ is a chain equivalence and  chain equivalences induce 
 isomorphisms on homology and cohomology (see \cite{Munkres}).
\qed
\end{proof}

For example, consider the cell complex $K'$ of Fig.~\ref{img:relations}. The $1$-cochain $\alpha^*$, defined  by the set $\{b,f\}$ of edges  
of $K'$, is a $1$-cocycle which `blocks' the white hole $H$ (in the sense that all the cycles representing the hole must contain an odd number of edges of $\alpha^*$). 
Then $\beta=\alpha^*\pi^r_1$ is defined by the set  $\{b,f,e\}$ of edges of $K$.
$\alpha$ and $\beta$ are both $1$-cocycles representing the same white hole $H$.

\section{Preserving Topology in Irregular Graph Pyramids}
\label{sec:preserving-topology}

Considering binary images, an \begin{it}object\end{it} is defined by a connected set (4-connectivity) of foreground pixels. 
\cR{A partition of the whole space (foreground and background) in cells is called  a \it{cell subdivision}. } 
The referred partition could be obtained from any of the planar graphs in every level of the pyramid.

 \cR{Fix a level $(G_i,\bar{G}_i)$, the} \textit{cell complex}   associated to the foreground object, called \textit{boundary cell complex}, denoted by $K^i$, is obtained 
\cR{from $(G_i,\bar{G}_i)$}
by taking all faces of $\bar{G}_{i}$ corresponding to vertices of $G_i$, whose receptive fields contain (only) foreground pixels, and adding all edges and vertices needed to represent the faces. The $p$-chain group generated by the $p$-cells of $K^i$ is denoted by $C_p(K^i)$.

\cR{The following lemma guarantees that the local operations applied to build a graph pyramid preserve the topological properties of the initial data.}

\begin{lemma}
\label{lem:pyr-homeomorphic}
The boundary cell complex is well-defined. All the boundary cell complexes of a given irregular dual graph pyramid are cell subdivisions of the same object. Therefore, all these cell complexes are homeomorphic.
\end{lemma}

\begin{proof} Our input is a binarized 2D digital image. An object is the $4$-connected set  of foreground pixels.
Since we only remove an edge in $\bar{G}_k$ when it is in the boundary of two different regions that have the same label (region merging), and contract an edge in $\bar{G}_k$ when it is incident to a vertex of degree $2$ or it is a pendant edge (simplification), all the new $p$-cells  created are homeomorphic to $p$-dimensional balls, $p=0,1,2$. 
\qed
\end{proof}

As a result of Lemma~\ref{lem:pyr-homeomorphic}, topological invariants computed on different levels of the pyramid are equivalent.

For the purpose of this paper, a new  \cR{cell complex} called \textit{homology-generator level}, is added over the 
\cR{boundary cell complex obtained from} the  \cR{top (last)} level of the pyramid (see Fig.~\ref{img:cocycle}).  \cR{This new  cell complex}  \cR{is denoted by $K^{\scst H}$ and it} is a set of regions surrounded by a set of self\cR{-}loops incident to a single vertex. To obtain this 
\cR{cell complex,}
 on the top of the computed pyramid, we compute a spanning tree \cR{of the boundary graph of the top level of the pyramid}, and contract all the edges that belong to it.  Note that $K^{\scst H}$ is no longer homeomorphic to any $K^i$, but homotopic.

\begin{figure}[tb]
\begin{center}
\includegraphics[width=14cm]{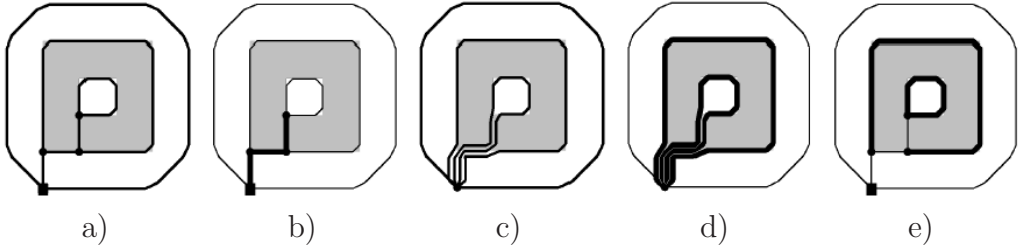} 
\\
a) \hspace{2.1cm} b) \hspace{2.1cm} c) \hspace{2.1cm} d) \hspace{2.1cm} e) \\
 \end{center}
\vspace{-2mm}
\caption{a) \cR{Boundary cell complex $K_4$ obtained from the top level of the  pyramid, $(G_4,\bar{G}_4)$;} b) in bold, spanning tree edges \cR{of $\bar{G}_4$;} c) \cR{homology-generator level, $K^{\scst H}$;} d) \cR{in bold, the self\cR{-}loops representing the cocycle edges in $K^{\scst H}$;} e) in bold, the cocycle edges in top level.}
\label{img:cocycle}
\end{figure}

\begin{lemma}
\label{lem:oper-chain-homotopy}
The two operations used to construct an irregular graph pyramid: edge removal and edge contraction, are
integral operators satisfying the chain-homotopy property.
\end{lemma}

\begin{proof}
\cR{Fix a level} $(G_i,\bar{G}_i)$, suppose an edge $e$ in $\bar{G}_i$ is removed. Since $\bar{G}_i$ is planar, then $e$ is in the boundary of two $2$-cells (or regions) $A$ and $B$ (see Fig. \ref{img:relations}). The integral operator $\phi^r$ associated to this edge removal is given by $\phi^r_1(e)=B$ (see Fig. \ref{img:relations1}).  
Now, suppose that an edge $d$ of $\bar{G}_i$, with a vertex $v$ of degree $2$ in its boundary, is contracted (see Fig. \ref{img:relations}).   The integral operator $\phi^s$ associated to this edge contraction is given by $\phi^s_0(v)=d$ (see Fig. \ref{img:relations1}).
\qed
\end{proof}

\cR{Starting from a cell decomposition of an object, its homology studies incidence relations of its subdivision. Cohomology arises from the algebraic dualization of the construction of homology. Both homology and cohomology are homotopy invariants.}

\begin{corollary}
The boundary cell complex of any level of the pyramid and the  homology-generator level
have isomorphic homology and cohomology groups.
\end{corollary}

As a consequence of  Lemmas \ref{lem:cocycle-integral-operator} and \ref{lem:oper-chain-homotopy} we have:

\begin{lemma}\label{lem:cocycle-pyr}
\cR{Fix a level} \cR{$(G_i,\bar{G}_i)$,} suppose an edge $e$ in $\bar{G}_i$, which is in the boundary of a region $B$, is removed. Let $\phi^r_1(e)=B$ be the integral operator associated to this removal. Let $\alpha^*$ be a $1$-cocycle defined by a set of edges $S$ in $\bar{G}_i\setminus\{e\}$. 
If an odd number of edges of $\alpha^*$ is in $B$, then $S\cup\{e\}$ defines a $1$-cocycle in $K^i$; otherwise, it is $S$ which defines a $1$-cocycle in $K^i$.
\end{lemma}

In terms of embedded graphs, an integral operator maps a vertex/point to exactly one of its incident edges and an edge to exactly one of its incident faces.
In every level of a graph pyramid, the contraction kernels make up a spanning forest. A forest composed of $k$ connected components, spanning a graph with $n$ vertices, has $k$ root vertices, $n-k$ other vertices, and also $n-k$ edges.
These edges can be oriented toward the respective root such that each edge has a unique starting vertex.
Then, integral operators mapping the starting vertices to the corresponding edge of the spanning forest can be defined as follows: $\phi_0(v_i) = e_j$, where 
$e_j$ is the edge incident to $v_i$, oriented away from it.

\cR{The following lemma guarantees that all integral operators that create homeomorphisms are in fact a combination of the two operations used to construct irregular graph pyramids.}

\begin{lemma} All integral operators that create homeomorphisms can be represented in a dual graph pyramid. This is equivalent to: given an input image $(G_0,\bar{G}_0)$ and its associated cell complex $Z = \{C_0, C_1, C_2\}$,  a cell complex $Z' = \{C_0', C_1', C_2'\}$ with $Z, Z'$ homeomorphic, and $Z$ a refinement of $Z'$ i.e. $C_0' \subseteq C_0,$ 
$C_1' \subseteq C_1,$ and $C_2' \subseteq C_2$, then there exists   a pyramid $P$ s.t. $Z'$ is the cell complex associated to some level  $(G_k,\bar{G}_k), k \geq 0,$ of $P$.
\end{lemma}

\section{Representative Cocycles in Irregular Graph Pyramids}
\label{sec:computation}

A method for efficiently computing representative cycles of homology generators using an irregular graph pyramid is given in \cite{Peltier09a}.
In \cite{IIKG08} a novel algorithm for correctly visualizing
graph pyramids, including multiple edges and self-loops is given. This algorithm preserves the geometry and the topology of the original image \cA{and has been used to produce the images throughout the paper} (see Fig.~\ref{img:PyramidDrawing})

\begin{figure}[tb]
\begin{center}
Level 0 \hspace{1.3cm} Level 1\hspace{1.3cm} Level 2\hspace{1.3cm} Level 3\hspace{1.3cm} Level 4\\
\includegraphics[width=14cm]{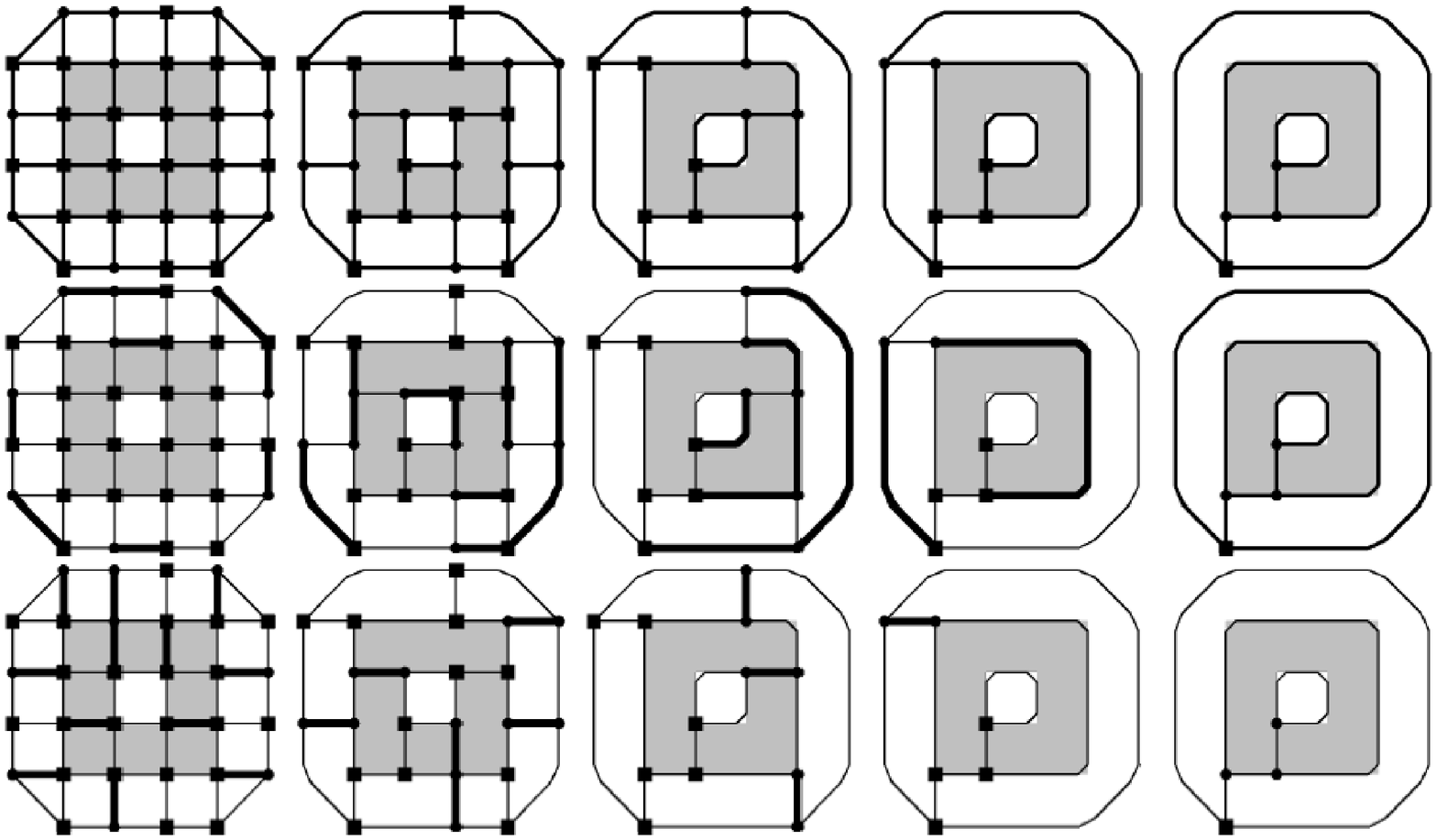}
\\
\end{center}
\caption{Top row, from left to right: boundary graphs for all levels of the pyramid. Vertices surviving to the next level are drawn with a square. Middle row, in bold: contracted edges in the respective levels. Bottom row: removed edges in the boundary graph, equivalent to contracted edges in the adjacency graph.
}
\label{img:PyramidDrawing}
\end{figure}

In this paper, representative cocycles are computed and drawn in the boundary graph of any level of a given irregular graph pyramid. They are computed in the top  level and down projected using the described process. \\

In the  homology-generator level (see Fig.~\ref{img:cocycle}.c),  each self-loop $\alpha$ that surrounds a region of the background (hole of a region $R$ of the foreground) is a representative $1$-cycle of a homology generator.

Let $K^{\scst H}$ be the  homology-generator level.
\cR{
Without loss of generality, we can suppose that $K^{\scst H}$ is connected. 
If not, repeat the following reasoning for each connected component (region) of $K^{\scst H}$.}
Let $\{\alpha_1,\dots,\alpha_n\}$ be the set of the self-loops surrounding a face of the background.  Therefore, there are $n$ white holes: $O_1,\dots O_n$ (see Fig.~\ref{img:cocycle}.d).
\cR{Fix $i$,} $i=1,\dots,n$, $\alpha_i$ is a representative $1$-cycle of the homology generator associated to the white hole  $O_i$.
Let $\beta$ be a self-loop
surrounding  \cR{the} face $f$ of the foreground \cR{(recall that we suppose that $K^{\scst H}$ is connected)} such that $\alpha_i$ is in the boundary of $f$ in $K^{\scst H}$.
Form  the sets $\{\alpha_1,\beta\},\dots,\{\alpha_n,\beta\}$. 
Let $K^0$ denote the boundary cell complex associated to the foreground in $\bar{G}_0$. Let $\{\phi_p:C_p(K^0)\to C_{p+1}(K^0)\}_{p\geq 0}$ be the composition of all integral operators associated with all  removals and contractions of edges of the foreground of the boundary graphs of a given irregular graph pyramid.  Let $\{\pi_p=id_p+\phi_{p-1}\partial_p+\partial_{p+1}\phi_p: C_p(K^0)\to C_p(K^{\scst H})\}_{p\geq 0}$ where
$\{\iota_p: C_p(K^{\scst H})\to C_p(K^0)\}_{p\geq 0}$ is the inclusion map.

\begin{prop}
The $1$-cochain $\alpha^*_i$ defined by the set $\{\alpha_i,\beta\}$ in $K^{\scst H}$ is a $1$-cocycle. Moreover, the set $\{\alpha_1^*,\dots,\alpha_n^*\}$ is a basis of representative $1$-cocycles.
\end{prop}

\begin{proof}
\cR{The set of edges of $K^{\scst H}$ is the set of the self-loops $\{\alpha_1,\dots,\alpha_n\}$ surrounding a region of the background together with the self-loop $\beta$, renamed by $\alpha_{n+1}$, surrounding the face $f$.}
 
\cR{The $1$-cochain $\alpha_i^*$ is a cocycle in $K^{\scst H}$ iff} $\delta_1(\alpha^*_i)=0$. Since we work with objects embedded in $\mathbb{R}^2$ then $\alpha_i$ can only be in the boundary of two faces. In this case, one face belongs to the background and the other face is $f$ in $K^{\scst H}$. Then $\delta_1(\alpha^*_i) (f)=\alpha^*_i(\partial_2(f))=\alpha^*_i(\alpha_i+\beta+\cdots)=\alpha^*_i(\alpha_i)+\alpha^*_i(\beta)+\alpha^*_i(\cdots)=1+1+0=0$.

\cR{Let us prove minimality. Suppose, for example, that $\alpha^*_1=\alpha^*_{j_1}+\cdots+\alpha^*_{j_s}$ where $1<j_1<\dots<j_s\leq n$, $s\geq 1$.
Then $\alpha^*_1$ is defined by the set $\{\alpha_{j_1},\dots,\alpha_{j_s}\}$ if $s$ even, and 
 $\{\alpha_{j_1},\dots,\alpha_{j_s},\beta\}$ if $s$ odd, which is a contradiction. }
\qed
\end{proof}

We will say that $\alpha^*_i$ is a representative $1$-cocycle of the cohomology generator associated to the white hole $O_i$.

\begin{algorithm}[tb]
\caption{Down project cocycle}\label{alg:downproj}

\cR{For each connected component (region) of $K^{\scst H}$.}
Let $A_{k},$ $k>0,$ denote the set of edges that define a cocycle in $\bar{G}_k$
(the boundary graph in level $k$).
The \textit{down projection} of $A_k$ to  $\bar{G}_{k-1}$ is the set of edges $A_{k-1} \subseteq \bar{G}_{k-1}$ that corresponds to $A_k$ i.e. represents the same cocycle. $A_{k-1}$ is computed as $A_{k-1}=A^s_{k-1}\cup A^r_{k-1}$, 
where $A^s_{k-1}$ denotes the set of surviving edges in $\bar{G}_{k-1}$ that
correspond to $A_{k}$, and 
$A^r_{k-1}$ is a subset of removed edges in $\bar{G}_{k-1}$. The following steps show how to obtain $A^r_{k-1}$:
\begin{enumerate}
\item Consider the contraction kernels of $G_{k-1}$ (RAG)
whose vertices are labeled with $\ell$ (the region for which cocycles are computed).
The edges of each contraction kernel are oriented
toward the respective root - each edge has a unique starting vertex.

\item For each contraction kernel $T$, from the leaves of $T$ to the root, let $e$ be an edge
of $T$, $v$ its starting point, and $E_v$ the edges in the boundary of the face associated to $v$:
label $v$ with the sum
of the number of edges that are in both $A^s_{k-1}$ and
in $E_v$, plus the labels of the children nodes of $v$.

\item A removal edge of $\bar{G}_{k-1}$ is in $A^r_{k-1}$ if
the starting point of the corresponding
edge of $G_{k-1}$ is labeled with an odd number.
\end{enumerate}
\end{algorithm}

Algorithm~\ref{alg:downproj} gives the proposed method to downproject a cocycle $\alpha^*$ from level $k$ to level $k-1$.
Informally, in the homology-generator level, there is only one face
representing the object. Based on the geometric interpretation of cocycles (Section~\ref{sec:cohom}), if we remove the edges and the face in-between, we destroy the hole. Then, there is no need to add any other edge to the cocycle to remove the hole. However, when going down in the pyramid, this face is partitioned. A connection among all the new regions is determined by the contraction kernels of the RAG. When the first partitioning occurs, the contraction kernel will contain one or two nodes corresponding to faces with one surviving cocycle edge in its boundary, and the rest of the nodes will have none. What Algorithm~\ref{alg:downproj} does is to \textit{find the unique path in the contraction kernel joining these two nodes, and take the set of boundary edges between consecutive faces as part of the new cocycle (see Prop.~\ref{prop:indep-surviving},~\ref{prop:indep-merging-order}, and their proofs)}. Lower levels will update the connections in subsections of the cocycle path. Every subsection will correspond to a partitioned region between two consecutive cocycle edges.

Consider the example in Fig.~\ref{img:cocycle}. In the homology-generator level we have $A_5$ = $\{\alpha,\beta\}$ the representative $1$-cocycle of a cohomology generator (self\cR{-}loops in Fig.~\ref{img:cocycle}.d). For down projection in  level \cR{$4$,} $A_4 =  A^s_{4}\cup A^r_{4}$. We have that $A^s_{4}$ is the surviving edges in bold of top level in Fig.~\ref{img:cocycle}.e). In this case, $A^r_{4} = \emptyset$ because  \cR{there} is no merging of foreground regions from \cR{the boundary cell complex obtained from the} top level to \cR{the homology-generator level.}

\begin{figure}[tbp]
\begin{center}

\includegraphics[width=14cm]{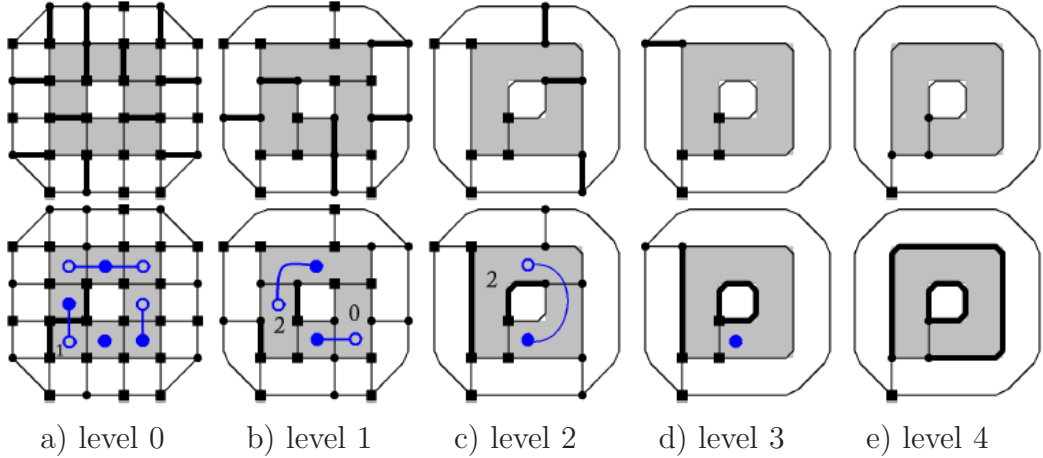}\\

a) level \cR{$0$}\hspace{.95cm} b) level \cR{$1$}\hspace{.95cm} c) level \cR{$2$}\hspace{.95cm} d) level \cR{$3$}\hspace{.95cm} e) level \cR{$4$}\\
 \end{center}
\caption{Top row: removed edges in boundary graph. Bottom row, from right to left: the down-projected cocycles in bold. Filled circles on faces, represent surviving vertices from the adjacency graph in foreground regions.}
\label{img:cocycleDP}
\end{figure}

In the example in Fig.~\ref{img:cocycleDP}, the cocycle $\alpha^*$ in level $4$  is the set of the two edges in bold (see Fig.~\ref{img:cocycleDP}, bottom row, column $d$).
The down projection from level $4$ to $3$ are the surviving edges of the cocycle in level $4$. This is because there was no contraction in the foreground region. The contractions of the adjacency graph can be seen in the top row of the figure. 

In level $2$ (Fig.~\ref{img:cocycleDP}, bottom row, column $c$), the first contraction of foreground in the adjacency graph with a single edge appears. In this case, the leaf node represents a face with an even number ($2$) of surviving cocycle edges in its boundary, which leads to not adding any other edge to the down-projected cocycle. Only in the base level (column a), one contraction kernel has a leaf node with an odd number ($1$) of surviving cocycle edges in the boundary. In this case, the corresponding edge in the boundary graph, for the edge connecting the respective node with its father, is added to the cocycle.

\begin{figure}[tbp]
\begin{center}

\includegraphics[width=14cm]{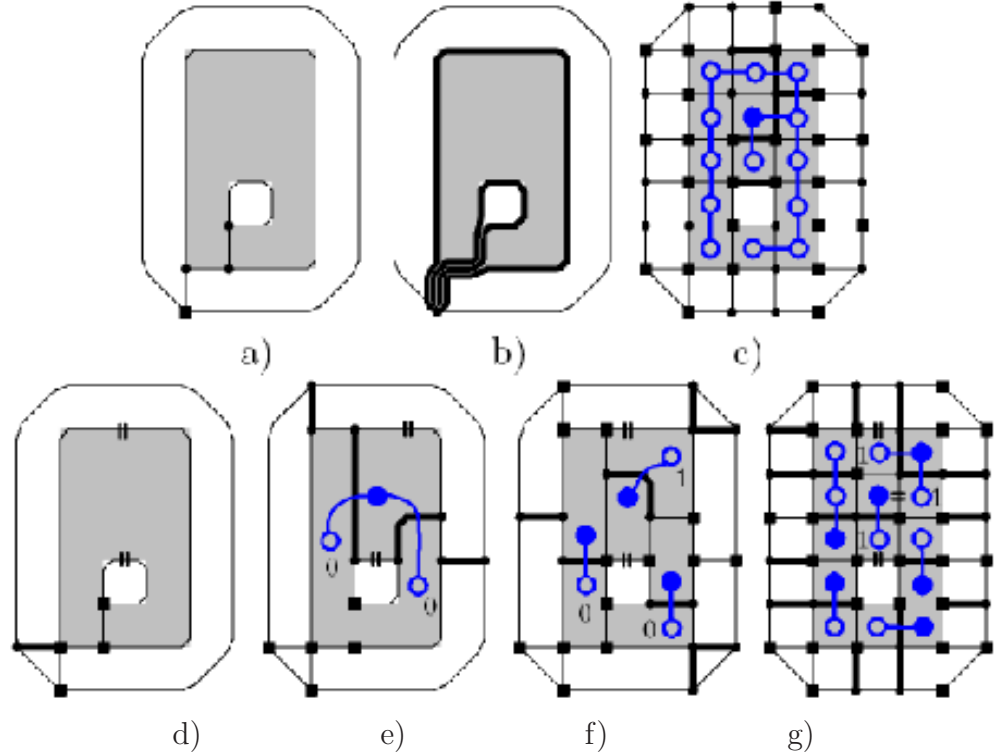} \\
d) \hspace{2.1cm} e) \hspace{2.1cm} f) \hspace{2.1cm} g)\\

 \end{center}
\vspace{-2mm}
\caption{a) \cR{Boundary cell complex obtained from the top level of the pyramid, $(G_4,\bar{G}_4)$;}
 b) \cR{homology-generator level} with cocycles edges in bold; c) ECK of the foreground in RAG, drawn over $\bar{G}_0$ with cocycle edges in bold; d) to g) shows removed edges in bold for levels from $3$ to $0$. Cocycle edges are marked with two small parallel lines.}
\label{img:sampleECK}
\end{figure} 

Any edge $\bar{G}_0$ that has survived to a higher level $k$, and was selected as part of the cocycle in $\bar{G}_k$, will belong to the \cR{down-}projected cocycle in $\bar{G}_0$. In particular, the edges in $\bar{G}_0$ that have survived to be the edges $\alpha,\beta$ of the cocycle in the homology-generator level, are going to be the \textit{entry} and \textit{exit} point of the cocycle path through the foreground region.

In Fig.~\ref{img:sampleECK}, the space between the outside boundary of the object and the hole is bigger, allowing for more possibilities for the paths of the cocycle. The cocycle path in the base is going to converge to the unique path connecting the surviving edges $\alpha,\beta$ through the ECK (see Fig.~\ref{img:sampleECK}.c).

In Fig.~\ref{img:sampleECK}.d the cocycle is made of the surviving cocycle edges from the homology-generator level in Fig.~\ref{img:sampleECK}.b~. The first partition of the foreground is connected by the \cA{contraction kernel} in Fig.~\ref{img:sampleECK}.e~. Here, one of the regions in the partition contains in its boundary the two surviving cocycle edges, so there is no path to find and no new cocycle edge to add. In Fig.~\ref{img:sampleECK}.f there is only one edge to add to the cocycle to connect the path of edges, which is identified with the leaf node with label 1. %

Notice that from level Fig.~\ref{img:sampleECK}.f to the one in  Fig.~\ref{img:sampleECK}.g also the selected `surviving' edges play a role, as edges to be removed at higher levels had to have been `surviving' ones at levels below. In this case, the edge we add to the cocycle in level $1$ (Fig.~\ref{img:sampleECK}.f), was a surviving edge in level below that was connecting two contraction kernels. Therefore, the selection of the surviving edges also determine the delineation of the down\cR{-}projected cocycle together with the contracted ones. The ECK contains the decision of which edges were contracted but also which ones were simplified, determining the unique path.%

\begin{prop}\label{downprojectionprop}
\begin{itemize}
\item[1. ] The down projection of $\alpha_i^*$ is a set of edges `blocking' the creation of the hole $O_i$, i.e., given a cycle $g$ homologous to the down projection of the cycle $\alpha_i$, the down projection of $\alpha_i^*$ contains an odd number of edges of $g$.   
\item[2. ] The down projection of $\alpha_i^*$ is always a cocycle. Moreover, the down projection of $\{\alpha_1^*,\dots,\alpha_n^*\}$ is a \cR{basis} of representative $1$-cocycles.
\end{itemize}
\end{prop}
 
\begin{proof}
\begin{itemize}
 \item[1. ]
 The down projection of $\alpha_i^*$, which is $\alpha_i^*\pi_1$, contains an odd number of edges of $g$ iff $\alpha_i^*(g)=1$. First, if $g$ is homologous to the down projection of $\alpha_i$, which is $\iota_1(\alpha_i)$, then there exists a $2$-chain $b$ in $K_0$ such that $g=\iota_1(\alpha_i)+\partial_2(b)$. Second,   $\alpha_i^*\pi_1(g)=\alpha_i^*\pi_1(\iota_1(\alpha_i)+\partial_1(b))=1$, since $\alpha_i^*\iota_1(\alpha_i)=\alpha_i^*(\alpha_i)=1$, and $\alpha_i^*\pi_1\partial_2(b)=0$ because $\alpha_i^*$ is a cocycle and $\pi_1\partial_2=\partial_1\pi_2$ (since $\{\pi_p:C_p(K^0)\to C_p(K*{\scst H})\}_{p\geq 0}$ is a chain equivalence \cite{Munkres}). So  $g$ must contain an odd number of edges of the set that defines  $\alpha_i^*$.
 \item[2. ] Proof of correctness of the down projection algorithm: it is a consequence of Lemma \ref{lem:cocycle-pyr}.
\qed
\end{itemize}
\end{proof}

Example down projections are shown in Fig.~\ref{img:PyramidDrawing}, \ref{img:cocycleDP}, and \ref{img:sampleECK}.

\subsection{Complexity}
Let $n$ be the height of the pyramid (number of levels) and \cA{$v_0, e_0$ the number of vertices, respectively edges} in the base level, with $n \approx \log v_0$ (logarithmic height).
An upper bound for the computation complexity is: $O(v_0 n)$ to build the pyramid; for each foreground component, $O(h)$ in the number of holes $h$ to choose the representative cocycles in the top level; $O(e_0  n)$ to down project \cA{each} cocycle. The overall computation complexity is then below $O(v_0 n + c (h e_0 n))$, where $c$ is the number of cocycles that are computed and down projected.

\cA{Actually not all edges are part of cocycles and not all levels have $e_0$ number of edges. When building $A_{k-1}$ one can go in linear time over the edges of $A_k$ and consider only the contraction kernels in $G_{k-1}$ for which the surviving vertices in $G_k$ correspond to one of the two faces to which an edge of $A_k$ is incident to. Then, computing $A_{k-1}$ actually takes $|A_{k}| + \sum_i{|T_{k-1}^i|}$ number of steps, where $T_{k-1}^i \in G_{k-1}$ are the contraction kernels mentioned before. Thus in practice the complexity of down projecting a cocycle is below $O(e_0 n)$.}

\begin{figure}[tbp]
\begin{center}
\includegraphics[width=14cm]{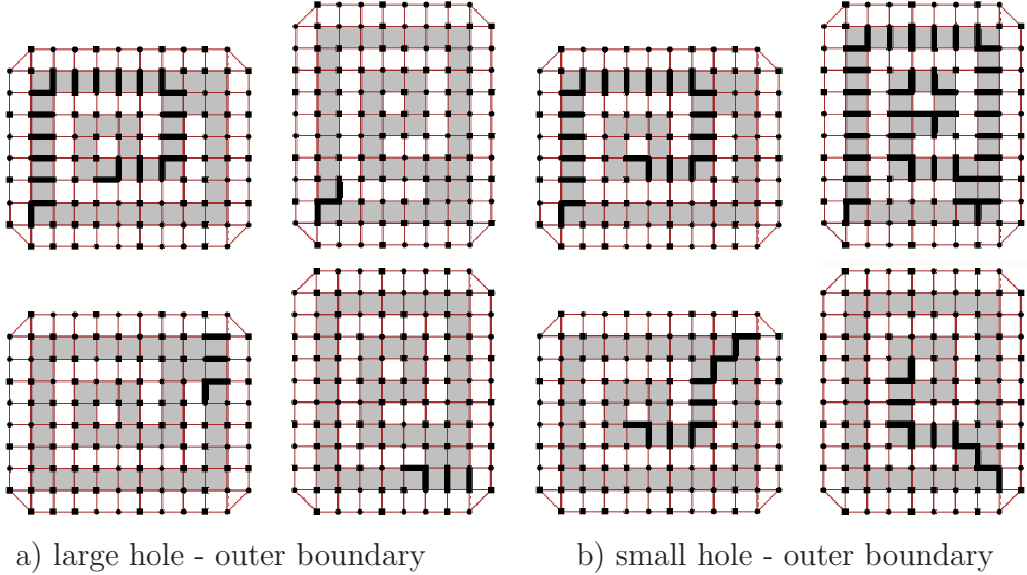}\\
a) large hole - outer boundary\hspace{2cm}b) small hole - outer boundary
\end{center}
\caption{Example showing (top) normal and (bottom) rotation invariant cocycles. The cocycles are down-projected starting with the surviving edge of the outer boundary and the surviving edge of the a) large and b) small holes.}\label{fig:rot-inv}
\end{figure}

\section{A First Step Towards Stable Cocycles}
\label{sec:stable-cocycles}

\cA{If topology is considered in the context of recognition, or a joint extraction of both topological and geometrical features is required, then the  location and shape of the extracted topological invariants becomes relevant and is an important way to ensure stability/repeatability. A relevant example is the work in~\cite{Dey2008} where handle and tunnel loops are made geometry aware by placing them on ``geometrically relevant'' positions.}

\cA{In this section we make a first stept towards obtaining stable cocycles and} consider invariance with respect to scanning and rotation of the object.
Because we start with a pixel grid and the \cR{$4$} neighborhood,\cA{only} rotations with multiple of $90$ degrees \cA{produce identical discretizations allowing for identical cocycles}.

The following properties are required for Prop.~\ref{prop:downproj-dependencies} which \cA{gives} the parts of the pyramid that the computed cocycles depend on.

\begin{prop}
\label{prop:only-removal}
The down-projected cocycles contain only removal edges in the boundary graphs, corresponding to edges in the ECK of the top vertex representing the object in $G_n$, and the two edges $\alpha, \beta$ that have survived to $\bar{G}_n$ and where selected as the cocycle $\alpha^*=\{\alpha, \beta\}$.
\end{prop}

\begin{proof}
Algorithm~\ref{alg:downproj} starts with two edges $\{\alpha, \beta\}$ in $\bar{G}_n$, and for each level $k = n-1, \dots, 0$ it adds only removal edges from $\bar{G}_k$ i.e. edges that where contracted in $G_k$ to merge neighboring regions belonging to the object.
\qed
\end{proof}

\begin{prop}
\label{prop:indep-surviving}
The result of down projection (Algorithm~\ref{alg:downproj}) does not depend on the selected surviving vertices in $G_k$.
\end{prop}

\begin{proof}
Consider the function $q : T \subseteq G_{k-1} \rightarrow \N$, $q(T) = \sum_{v \in T} |A^s_{k-1} \cap E_v|$ (see Alg.\ref{alg:downproj} for the used notation).
Every cocycle $A_k$ has an even number of edges from the boundary of any face in $\bar{G}_k$ (Section~\ref{sec:cohom}).  Then $q(T)$ is also even and the number of vertices $v$ for which $|A^s_{k-1} \cap E_v|$ is odd, is even. For any edge $e \in T = (V, E)$ consider the two connected components (trees) $T_1, T_2$ of the subgraph $(V, E\setminus\{e\})$ ($e$ is a cut edge of $T$ because $T$ is a tree). The removal edge of $\bar{G}_{k-1}$ corresponding to the edge $e$ is added to the cocycle if $q(T_1)$ and $q(T_2)$ are even, which is independent of the originally chosen surviving vertex, the root of $T$.
\qed
\end{proof}

\begin{figure}[t!]
\begin{center}
\includegraphics[width=14cm]{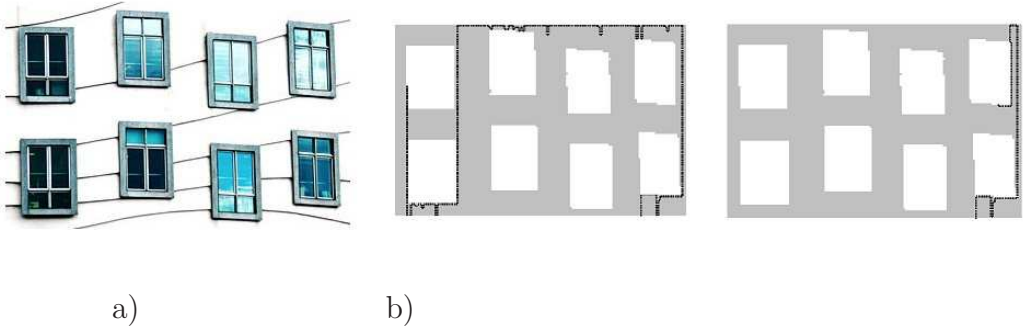}
$\mbox{ }$\\
 a) \hspace{3cm} b) \hspace{3cm} \hspace{3cm}$\mbox{ }$ \\
\end{center}
\caption{a) original image; b) in bold, the paths in the RAG $G_0$ associated to the down-projected cocycles related to the holes representing the top-left and top-right windows.}\label{fig:windows}
\end{figure}
\begin{prop}
\label{prop:indep-merging-order}
A down\cR{-}projected cocycle does not depend on the order in which edges are removed in the boundary graphs (region merging).
\end{prop}

\begin{proof}
Consider the proof of Prop.~\ref{prop:indep-surviving}. What Algorithm~\ref{alg:downproj} does is to select additional edges from the contraction kernel $T$, to connect the vertices $v$ with odd $|A^s_{k-1} \cap E_v|$. In a tree, there is a unique path connecting any two vertices. Denote by $K$ the ECK of the vertex in $G_n$ corresponding to the object. Algorithm~\ref{alg:downproj} returns $\alpha$ and $\beta$, plus the set of edges of $\bar{G}_0$ corresponding to the path in $K$ that connects the two vertices whose corresponding faces in $\bar{G}_0$ have $\alpha$ and $\beta$ in their boundary. The ECK of a vertex does not depend on the order of the intermediate steps~\cite{Kropatsch97}.
\qed
\end{proof}

The following property is an immediate result of Properties~\ref{prop:only-removal}, \ref{prop:indep-surviving}, and~\ref{prop:indep-merging-order}.

\begin{prop}
\label{prop:downproj-dependencies}
The cocycles computed by Algorithm~\ref{alg:downproj} depend only on the cocycle $\{\alpha,\beta\}$ chosen in the top level\cR{,} $(G_n,\bar{G}_n)$, and on the ECK of the vertex in $G_n$ corresponding to the face describing the object in $\bar{G}_n$.
\end{prop}

The following property results from Prop.~\ref{prop:downproj-dependencies} and motivates the modification proposed in the rest of this section.

\begin{prop}
\label{prop:scan-invariance}
If the ECK of the vertex representing the object in the adjacency graph of the homology\cR{-}generator level, and the edges that survive to be in the boundary of the corresponding face in the boundary graph, are scanning and rotation invariant, we will obtain scanning and rotation invariant cocycles.
\end{prop}

In the following we will consider the necessary additions to the pyramid building process, to ensure that computed cocycles do not depend on the scanning and rotation of the object. We follow Prop.~\ref{prop:scan-invariance} and consider the ECK and the boundary edges of the top level.

\begin{figure}[t!]
\begin{center}
\includegraphics[width=14cm]{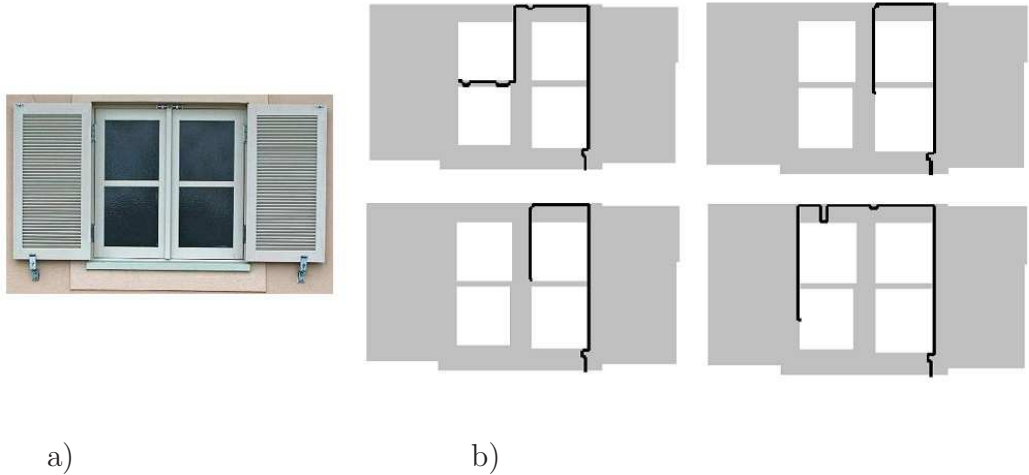}
$\mbox{ }$\\
a) \hspace{5cm} b) \hspace{3cm} \hspace{3cm}$\mbox{ }$ \\
\end{center}
\caption{a) original image; b) in bold, the paths in the RAG $G_0$ associated to the down-projected cocycles. 
}
\label{fig:window}
\end{figure}

\subsection{Invariant ECK and surviving edges}

As the edges to be removed in $\bar{G}_k$ are `locally' chosen, in a binary image like the ones used to represent our objects, there is no local structure and a random or scanning/orientation dependent direction is taken (Line~\ref{lin:boru} of Algorithm~\ref{alg:pyramid}).

To `add structure' and create an ordering for selecting edges to be removed and contracted:
\begin{enumerate}
 \item Compute a \textit{spanning tree} of the subgraph $O \subset G_0$ corresponding to our object in the base level. Mark the \textit{edges of $\bar{G}_0$ corresponding to the edges of the spanning tree to be the removed} ones.
 \item Create a \textit{strict ordering between any two edges}. This ordering is used to select surviving edges during simplification, and thus controls the choice of edges in the homology\cR{-}generator level.
\end{enumerate}

\paragraph{The tree} Given the graph $O$ corresponding to our object, and a vertex $s \in O$, we define $d(v) : O \rightarrow \N$ to be the number of edges of the shortest path connecting $v$ and $s$ in $O$ i.e. the geodesic distance between the two pixels corresponding to $v$ and $s$, using the 4 neighborhood.
Note that a vertex $s$ can be obtained in a rotation invariant manner for example by using an automatic shape orientation method~\cite{ZunicKF06} and then selecting the top, left-most vertex.

\paragraph{Stable ECK (in $G_0$)} Every vertex $v \in G_0$, $v \neq s$, labels the edge $(v,v') \in G_0$ with $d(v') = d(v)-1$ as `to contract'. If $v$ has more than one neighbor $v'$ with $d(v') = d(v)-1$ the neighbor that minimizes the angle $\widehat{SVV'}$, and in case of angle equality, the one that has a clockwise orientation of $SVV'$ is chosen. $S$, $V$, and $V'$ are the points in $\Z^2$ corresponding to the centers of the pixels represented by $s$, $v$, and $v'$. The edges in $\bar{G}_0$ corresponding to the edges of $G_0$ labeled as `to contract', are marked as `to remove'.

\paragraph{Region boundary simplification (in $\bar{G}_k$)} In Line~\ref{lin:contract-and-simplify}.iii) of Algorithm~\ref{alg:pyramid}, from any chain of edges bounded by at least one vertex of degree 2, one edge will survive and all others will be contracted. To choose to surviving edge, assign to each edge $e \in \bar{G}_0$ bounding a cell of the object, the value $f(e) = min\{d(v_1),d(v_2)\}$ where $v_1$, $v_2$ are the vertices of $G_0$ corresponding to the two faces of $\bar{G}_0$ to which $e$ is incident to. Faces not part of the object are ignored.
When choosing the edge to survive i.e. not contract, the edges are sorted using the following (transitive) relation between any two edges $e$ and $e'$:
\begin{itemize}
 \item $f(e)$ vs. $f(e')$;
 \item if $f(e)=f(e')$ then use the orientation of $c(e)\,S\,c(e')$ vs. the orientation of $c(e')\,S\,c(e)$, where $c(e),c(e') \in \R^2$ are the centers of the edges $e$ and $e'$, and $S \in \Z^2$ is the center of the pixel used to define the rotation invariant tree;
 \item if $c(e)\,S\,c(e')$ are collinear, the Euclidean distance between $c(e)$ and $S$ vs. the Euclidean distance between $c(e')$ and $S$.
\end{itemize}

\paragraph{Homology\cR{-}generator level} When building the homology\cR{-}generator level, all edges of $\bar{G}_{n-1}$ bounding the face of the object are sorted based on the criteria above. Edges are selected in inverse order and used to create the spanning tree to be contracted. Edges not bounding the face corresponding to the object are added in random order.

Fig.~\ref{fig:rot-inv} shows an example object and its computed cocycles with and without the rotation invariant pyramid.
Fig.~\ref{fig:windows} and ~\ref{fig:window} show the paths in the RAG associated to the down-projected cocycles of the test images in \cite{Peltier09a}. Fig.~\ref{fig:camiseta}  is another example showing the path in the RAG $G_0$ associated to the down-projected cocycle related to the hole associated to the ball. Finally, Fig~\ref{fig:iniesta} 
shows the down-projected cocycles computed on a image from the 2010 World Cup final in South Africa.

\begin{figure}[t!]
\begin{center}
\includegraphics[width=14cm]{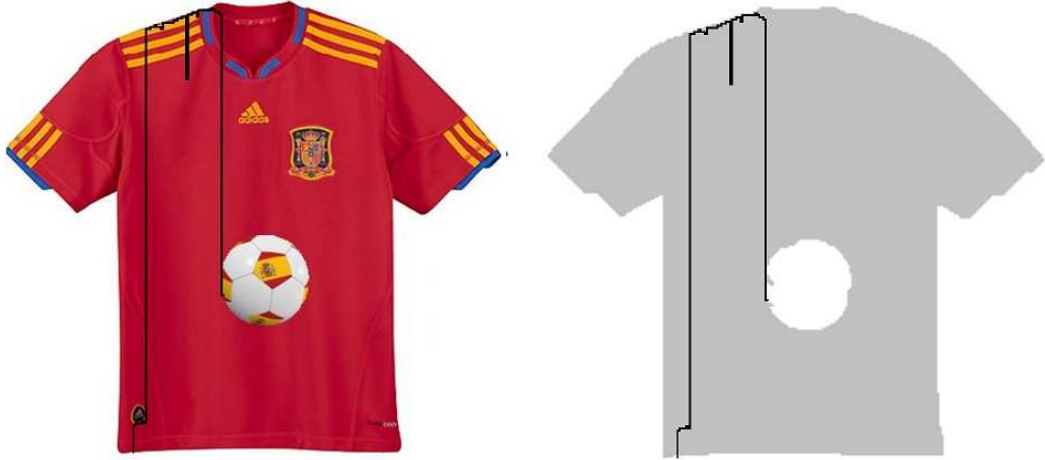}
\end{center}
\caption{In bold, the path in the RAG $G_0$ associated to the down-projected cocycle.}
\label{fig:camiseta}
\end{figure}

\subsection{Discussion}
Besides stability with respect to geometric transformations, one could consider additional criteria like the minimality of the obtained cocycles w.r.t. additional measures, like for example the number of edges of each cocycle, the sum of the number of edges of a basis of representative cocycles, the length of the path in $\R^2$ passing through the support squares of the pixels having at least one cocycle in their boundary\footnote{Would give Euclidean rotation robust cocycles.}, etc.

It has been shown for homology generators~\cite{ChenF10} that in general the problem of computing minimal representative cycles is NP-hard. Nevertheless, in certain cases, like computing $(n-1)$-cycles for $n$ dimensional objects, finding minimal cycles is not NP-hard.
For cocycles such a study does not exist yet, but considering the relation between homology and cohomology, similar results can be expected. In the case of 2D objects, the problem of finding cocycles with minimal number of edges can be related to the problem of finding in the RAG shortest paths that connect vertices adjacent to different holes  -- this problem can be solved in $n \log n$ if using Dijkstra's algorithm~\cite{Atallah98a}.

Real life objects are typically obtained by using scanning devices of different type: 3D scanners, video cameras, CT, MRI, etc. One common issue in all these cases is the presence of noise and being robust w.r.t. the possible deformations of a real object. A possible solution could be to define a robust basis of cocycles based on a function like the eccentricity transform~\cite{Ion08a} which is known to be robust w.r.t. noise and deformations. \cA{The eccentricity transform associates to each point of a shape the geodesic distance to the point furthest away. For a given starting point the geodesic distance function defines behind holes a set of points called the \textit{cut locus} which can be reached in the same distance on multiple paths (going on both sides of the hole). In many cases, cutting a shape along these sets can produce a shape with less holes, which gives the same geodesic distance function for the same starting point. The eccentricity transform can be interpreted as the maximum over multiple geodesic distance propagations initiated at each point of the shape. As cocycles can be seen as `cuts into the object' that `kill' a hole, a cut which has the minimum effect on the eccentricity transform, could provide an avenue for selecting robust cocycles.}

\cA{Finding associations between concepts in cohomology and graph theory will open the door for applying existing efficient algorithms (e.g. shortest path). The following lemma can be seen as a first step in this direction.
\begin{lemma} \label{holeoutside} Any set of foreground edges in the boundary graph $\bar{G}_0$, associated to a path in the RAG $G_0$, connecting a hole $O_i$ of the object with the (outside) background face, is a representative 1-cocycle 
cohomologous to the down projection of the $1$-cocycle $\alpha^*_i$. It blocks any generator that would surround the hole.
\\
In other words, consider the down projection~\cite{Peltier09a} of $\alpha_i$ and $\beta$ in $\bar{G}_0$: the $1$-cycles $\iota_1(\alpha_i)=a$ and $\iota_1 (\beta)=b$, respectively. Take any edge $e_a \in a$ and $e_b \in b$. Let $f_a$, $f_b$ be faces of  $K^0$, the boundary cell complex associated to the foreground in $G_0$ having $e_a$, respectively $e_b$, in their boundary. Let $v_0,v_1,\dots,v_n$ be a simple path of vertices in $G_0$ s.t. all vertices are labeled as foreground. $v_0$ is the vertex associated to $f_a$, and $v_n$ to $f_b$.
Consider the set of edges $c = \{e_0,\dots,e_{n+1}\}$ of $\bar{G}_0$, where $e_0=e_a$, $e_{n+1}=e_b$, and  $e_{\ell}$, $\ell = 1\dots n$, is the common edge of the regions in $\bar{G}_0$ associated with the vertices $v_{i-1}$ and $v_i$. $c$ defines a $1$-cocycle cohomologous to the down projection of the $1$-cocycle $\alpha^*_i$.
\end{lemma}
\begin{proof}
  $c$ is a $1$-cocycle iff $c\partial_2$ is the null homomorphism.
First, $c\partial_2(f_{\ell})=c(e_{\ell}+e_{\ell+1})=1+1=0$. Second, if $f$ is a $2$-cell of $K^0$, $f\neq f_{\ell}$, $\ell=0,\dots, n$, then, $c\partial(f)=0$. 
To prove that the cocycles $c$ and $\alpha_i^*\pi_1$ (the down projection of $\alpha_i^*$ to the base level of the pyramid) are cohomologous, is equivalent to prove that $c\iota_1=\alpha_i^*$. We have that  $c\iota_1(\alpha)=c(e_b)=1$ and $c\iota_1(\beta)=c(e_a)=1$. Finally, $c\iota_1$ over the remaining self-loops of the boundary graph of the homology-generator level is null. Therefore, $c\iota_1=\alpha_i^*$. 
\qed 
\end{proof}
}

\begin{figure}[t!]
\begin{center}
\includegraphics[width=14cm]{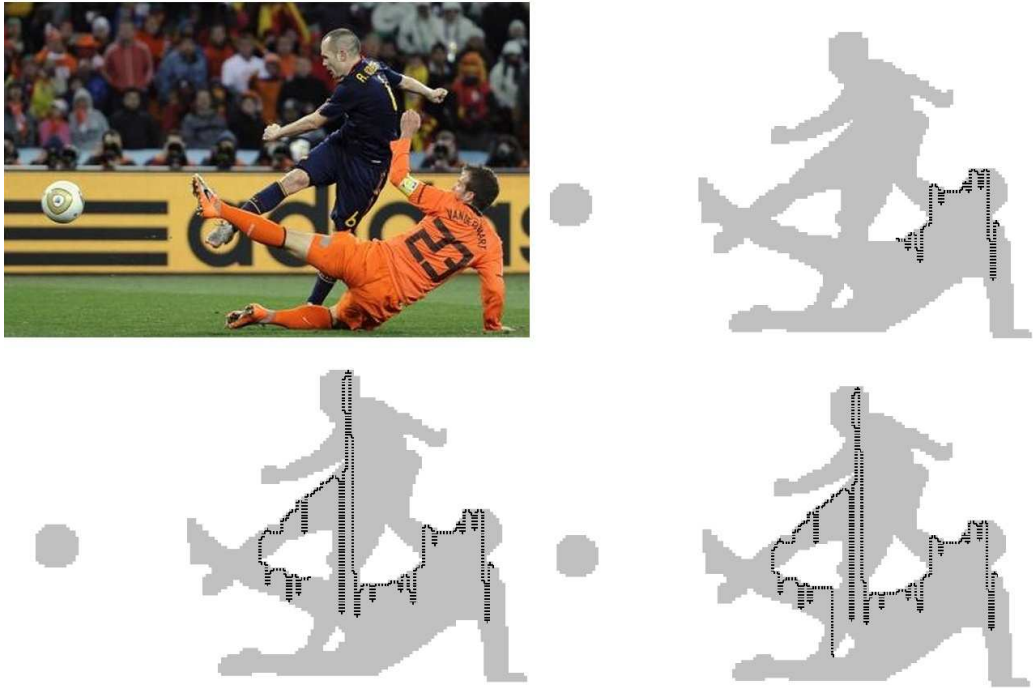}
\end{center}
\caption{On the top-left, original image (a image from the 2010 World Cup final in South Africa); in bold, down-projected cocycles.}
\label{fig:iniesta}
\end{figure}


\section{Conclusion}
\label{sec:conc}

This paper considers cohomology in the context of graph pyramids. Representative cocycles are computed at the reduced top level and down projected to the base level corresponding to the original image. Connections between cohomology and graph theory are proposed, considering the application of cohomology in the context of classification and recognition. The current paper extends the previous work with detailed insights and proofs, and a refinement of the previous method that makes the obtained cocycles scanning and rotation invariant. Extension to higher dimensions, where cohomology has a richer algebraic structure than homology, and complete cohomology - graph theory associations are proposed for future work. \cR{For this last task, we could consider nD generalized map pyramids \cite{Damiand06}.}


\section{Acknowledgements}
This work was partially supported by the Austrian Science Fund under grants S9103-N13, P18716-N13 and P20134-N13.


\section*{References}

\bibliographystyle{elsarticle-num}

\bibliography{gbr2cviu}

\end{document}